\documentclass[accepted]{uai2021} % after acceptance, for a revised
\usepackage[american]{babel}
\usepackage{mathtools} % amsmath with fixes and additions
\usepackage{booktabs} % commands to create good-looking tables
\usepackage{tikz} % nice language for creating drawings and diagrams

\usepackage{amssymb}
\usepackage{bm}
\usepackage{amsmath,amsthm,dsfont}
\newcommand{\norm}[1]{\left\lVert#1\right\rVert}

\newcommand{\Lim}[1]{\raisebox{0.5ex}{\scalebox{0.8}{$\displaystyle \lim_{#1}\;$}}}

\newtheorem{lemma}{Lemma}
\newtheorem{theorem}{Theorem}
\newtheorem{remark}{Remark}

\newtheorem{conjecture}{Conjecture}
\newcommand{\comment}[1]{}

\title{Statistical Mechanical Analysis of Neural Network Pruning}

\author[1]{Rupam Acharyya}
\author[2]{Ankani Chattoraj\thanks{ equal contribution}}
\author[1]{Boyu Zhang$^*$}
\author[3]{Shouman Das}
\author[1]{Daniel \v{S}tefankovi\v{c}}
\affil[1]{%
    Computer Science Dept.\\
    University of Rochester\\
    Rochester, New York, USA
}
\affil[1]{%
    Brain anf Cognitive Science Dept.\\
    University of Rochester\\
    Rochester, New York, USA
}
\affil[1]{%
    Mathematics Dept.\\
    University of Rochester\\
    Rochester, New York, USA
}

\begin{document}
\maketitle

\begin{abstract}
Deep learning architectures with a huge number of parameters are often compressed using \textit{pruning} techniques to ensure computational efficiency of inference during deployment. Despite multitude of empirical advances, there is a lack of theoretical understanding of the effectiveness of different pruning methods. We inspect different pruning techniques under the statistical mechanics formulation of a teacher-student framework and derive their generalization error (GE) bounds. It has been shown that \textit{Determinantal Point Process} (DPP) based \textit{node} pruning method is notably superior to competing approaches when tested on real datasets. Using GE bounds in the aforementioned setup we provide theoretical guarantees for their empirical observations. Another consistent finding in literature is that sparse neural networks (\textit{edge pruned}) generalize better than dense neural networks (\textit{node pruned}) for a fixed number of parameters. We use our theoretical setup to prove this finding and show that even the baseline \textit{random edge pruning} method performs better than the \textit{DPP node pruning} method. We also validate this empirically on real datasets.
\end{abstract}

\section{Introduction}
Deep neural networks have achieved impressive results in a wide variety of applications such as classification \cite{krizhevsky2012imagenet,liu2017survey}, image processing \cite{litjens2017survey, badrinarayanan2017segnet}, natural language processing \cite{devlin2018bert,deng2018deep,socher2013recursive}, etc. Most of these networks use millions and sometimes even billions of parameters which makes inference computationally expensive and memory intensive \cite{devlin2018bert}. To address this, researchers explore pruning techniques with the primary goal of comparing performance on real datasets. The broad scientific paradigm explored by most pruning techniques is to empirically and heuristically determine either how to prune a network or what to prune in a network (sometimes both). In this work, we take a step towards theoretical understanding of these two prime aspects of pruning methods. 

We compare the quality of different pruning methods for feedforward neural networks under the \textit{teacher-student} framework \cite{saad1995exact,saad1995line,saad1997learning,goldt2019dynamics} in the thermodynamic limit (input dimension goes to infinity) using \textit{generalization error bounds} (GE), a theoretical measure of performance of machine learning models on unseen test data \cite{vapnik1999overview}. 

A fairly recent work by \cite{mariet2015diversity} empirically investigates a node pruning technique where a diverse subset of nodes are preserved in a given layer using Determinantal Point Process (DPP) \cite{macchi1975coincidence,kulesza2012determinantal}. We provide theoretical guarantees for their empirical observations thereby showing that DPP based node pruning outperforms two standard paradigms of pruning (magnitude based node pruning and random node pruning). Thus, in the first part of this paper, we take a step towards theoretical understanding of the question: how  to prune? 

For the second part of this work we focus our attention to the study by \cite{blalock2020state}. This study reviewed multiple papers across decade on various pruning methods and closely analyzed their empirical results to conclude that sparse models obtained after edge/connection (used interchangeably) pruning outperforms dense ones obtained after node pruning for a fixed number of parameters. We extend our theoretical setup and compare node and edge pruning techniques which are within the scope of our investigation, to provide a theoretical justification of their empirical observation driven claim, thereby addressing the question: what to prune? 

Our work has multiple contributions with regard to theoretical advancements in the domain of pruning: 
\begin{itemize}
    \item We use GE bounds on the teacher-student framework to compare different pruning methods within a class, which to the best of our knowledge, is the first theoretical advance in comparing pruning methods.
    \item We prove that DPP node pruning outperforms random and importance node pruning methods, previously shown by \cite{mariet2015diversity} empirically.
    \item We also theoretically show and validate on real datasets (\texttt{MNIST} and \texttt{CIFAR10}), that baseline random edge pruning performs better than DPP node pruning (superior in the node pruning regime explored in this paper) which is consistent with empirical observations from pruning literature that sparse models outperform dense models \cite{blalock2020state}.
\end{itemize}

\section{Related Work}
\textbf{Pruning Methods:} Studies under node pruning regime remove entire neurons/nodes (used interchangeably henceforth) keeping the networks dense \cite{he2014reshaping,li2016pruning,he2017channel}. Our work is closely related to \cite{mariet2015diversity}, where a DPP sampling technique is used to select a set of diverse neurons/nodes to be preserved during pruning. The authors also introduce a \textit{reweighting} procedure to compensate contributions of the pruned neurons in the network. Finally, they compare \textsc{DIVNET} (DPP node pruning with reweighting as in \cite{mariet2015diversity}) with random and importance node pruning \cite{he2014reshaping} on real datasets. Seminal studies on edge pruning \cite{lecun1990optimal,hassibi1993second} remove unimportant network weights based on the Hessian of the network's error function. Among others, alternative approaches include low rank matrix factorization of the final weight layers \cite{sainath2013low} or pruning the unimportant connections below a threshold \cite{han2015deep}. Though dense networks can benefit from modern hardware, sparse models outperform dense ones for a fixed number of parameters across domains \cite{lee2019signal,kalchbrenner2018efficient,gray2017gpu}. In a recent review this is highlighted based on observations from investigating $81$ studies on pruning techniques \cite{blalock2020state}.  
% A recent article \cite{blalock2020state} rigorously reviews numerous papers in the domain of pruning to establish the present status of the field. 

The various existing methods can be broadly subsumed into a couple of categories \cite{blalock2020state}. These categories are mainly governed by the principles of pruning heuristics. First category is the magnitude-based approaches which have been extensively studied both globally and layerwise \cite{han2015deep,gale2019state}. As per \cite{blalock2020state}, magnitude-based approaches are not only good and common baselines in the literature but they also give comparable performance to other methods such as the gradient-based methods \cite{lee2019signal,yu2018nisp}. Another category is the random pruning which serves as an useful baseline for showing superior performance of any other pruning technique. We hence show all our theoretical results w.r.t these two categories, random pruning and importance pruning (same in concept as magnitude based pruning). We do not focus on any specific algorithm within these categories but explore the general concept for theoretical results. There are recent advances in pruning techniques which are complementary to these approaches, such as, being data independent \cite{ben2020data, tanaka2020pruning}, single shot \cite{lee2018snip, van2020single} etc. However, these are beyond the scope of our investigation.

\textbf{Theoretical Advances Towards Understanding Neural Networks:} Despite promising performance in empirical data, providing theoretical guarantees for neural networks remains a known challenge. Researchers have explained the training dynamics of neural network from the information theoretic perspective \cite{tishby2015deep, saxe2019information}. In another direction of work the learning dynamics of neural networks with infinitely wide hidden layers are explored \cite{jacot2018neural,du2018gradient,arora2019exact,zou2020gradient}. Pioneering work by \cite{saad1995exact,saad1995line,saad1997learning} analyzes the generalization dynamics form the statistical mechanics perspective on \textit{teacher-student} framework~\cite{gardner1989three} to understand the performance of neural networks on unseen test data. All our theoretical analyses throughout this work closely follow \cite{advani2017high, goldt2019dynamics}, who analyzed results for the case where the student networks are over parameterized, i.e., it has more number of hidden nodes than the teacher network.

\section{Preliminaries}\label{prelims}
\textbf{Determinantal Point Process (DPP):} DPP~\cite{macchi1975coincidence} is a probability distribution over power set of a ground set $\mathcal{G}$, here finite. DPP is a special case of negatively associated distributions \cite{joag1983negative} which assigns higher probability mass on diverse subsets. Formally, a DPP with a marginal kernel $L$ ($\in \mathbb{R}^{|\mathcal{G}| \times |\mathcal{G}|}$) is:
$\mathbb{P}[\mathbf{Y}=Y] = \frac{\det (L_Y)}{\det (L+I)}$, where $Y \subseteq \mathcal{G}$ and $L_Y$ is the principal submatrix defined by the indices of $Y$. We use $k$-DPP to denote the probability distribution over subsets of fixed size $k$. 

\textbf{DPP Node Pruning:} \cite{mariet2015diversity} uses DPP to propose a novel node pruning method for feedforward neural network. They define information at node $i$ of layer $l$ as $\bm{a}_i^l (=(a_{i1}^l,\ldots,a_{in}^l))$, where $a_{ij}^l$ is the activity of node $i$ of layer $l$ on $j^{th}$ input. Here $\bm{a}_i^{l} = g\left( \bm{b}_i^l \right)$, where $ \bm{b}_i^l = \sum_{j=1}^{n_{l-1}} w_{ji}^{l-1} \bm{a}_j^{l-1}$ is the information at node $i$ of layer $l$ before activation. A layer is pruned by choosing a subset of hidden nodes using a DPP kernel: $\bm{L}$ ($ =  \bm{L'} +\epsilon \bm{I}$), where, $\bm{L'}_{st} = \exp( -\beta \norm{  \bm{a}_s^l -  \bm{a}_t^l }^2 )$ and $\beta$ is a bandwidth parameter. The matrix $\bm{L}$ is of dimension $n_l \times n_l$, as total number of nodes in layer $l$ is $n_l$. By the property of DPP, this procedure will keep a diverse subset of nodes for each layer w.r.t. information obtained from the training data. A \textit{reweighting} technique (see Section 2.2 of \cite{mariet2015diversity}) is then applied to outgoing edges of retained nodes to compensate for information lost in that layer due to node removal. 

\textbf{Remark:} \textsc{DIVNET} denotes DPP node pruning with reweighting as in \cite{mariet2015diversity}.

\begin{table*}[t]
  \caption{Notations used in Theorems}
  \begin{center}
  \small
\scalebox{0.9}{
  \begin{tabular}{|c|c|c|c|c|c|c|}
    \hline
    \textbf{Notations}& \textbf{Explanations} & \textbf{Notations} & \textbf{Explanations} & \textbf{Notations} & \textbf{Explanations}  \\
    \hline
     $n$ & number of inputs &  $N$ & dimension of the input &   $n_l$ & number of nodes in layer $l$\\
    \hline
    $v_{i}^l$ &  $i^{th}$ node in layer $l$ $(1\leq i \leq n_l)$ &
    $a_{ij}^l$ & activation of $v^{l}_i$ on $j^{th}$ input & $M$ &  number of teacher\\
    &&&&& hidden nodes\\
    \hline 
    $e_{ij}^l$ & edge from  $v_{i}^l$ to $v_{j}^{l+1}$    & $w_{ij}^l$ & weight of  $e_{ij}^l$ & $K$ & number of student\\
     &  $(1\leq i \leq n_l$ and $1\leq j \leq n_{l+1})$ & & $(1\leq i \leq n_l$ and $1\leq j \leq n_{l+1})$ & & hidden nodes \\
    \hline 
    $k_n$ &  number of student hidden nodes &  $k_e$
     &  number of incoming edges of a  & $v^*$ & second layer weight\\
     & kept after node pruning &  & hidden node kept after edge pruning  & & of teacher network\\
     \hline
\end{tabular}}
\end{center}
\label{tab:notations}
\end{table*}

\textbf{Online Learning in Teacher-Student Setup~\cite{goldt2019dynamics}:}\label{teach-stud} We use a two-layer perceptron which has $N$ input units, $M$ hidden units and 1 output unit as the \textit{teacher network} to generate labels for i.i.d Gaussian input, $\bm{x}^t= (x_1^t,\ldots,x_N^t)$ where $x_i^t \sim \mathcal{N}(0,1)$ $\forall i \in \{1,\ldots,N\}$. Let $\theta^* = \{\bm{w}^*(\in \mathbb{R}^{M \times N}), v^* \in \mathbb{R}^M\}$ denote the fixed parameters of the teacher network. The label $y^t$ of the input $\bm{x}^t$ ($t = 1,2,\ldots$) is given as, 
\begin{equation}
\label{label-teacher}
y^t = \sum_{m=1}^M v^*_m g\left(\frac{w^*_m \bm{x}^t}{\sqrt{N}}\right) + \sigma \zeta^t,
\end{equation}
where $\zeta^t \sim \mathcal{N}(0,1)$ is the output noise, and $g$ is the sigmoid activation function. The input and teacher generated labels ($\{(\bm{x}^1,y^1),\ldots\}$) are used to train a two-layer \textit{student network} with $N$ input units, $K$ hidden units ($K \geq M$) and 1 output unit using online SGD learning method. We consider the quadratic training loss, i.e., 
\begin{equation}\label{MSE-loss}
L(f) = \frac{1}{2} \left[ \sum_{k=1}^K v_k g\left(\frac{w_k \bm{x}^t}{\sqrt{N}}\right)  - y^t \right]^2,
\end{equation}
where $f = \{\bm{w},\bm{v}\}$ denotes the parameter of the student network. One of the key quantities for evaluating performance of neural network is \textit{generalization error} (GE). For the teacher student setup the GE with teacher network $f^*$ and student network $f$ is denoted as $\epsilon(f,f^*)$. It is defined as,
$$\epsilon(f,f^*) = \frac{1}{2}\langle [\phi(x,f) - \phi(x,f^*)]^2  \rangle,$$
where $\phi(x,f) = \sum_{k=1}^K v_k g\left(\frac{w_k \bm{x}^t}{\sqrt{N}}\right)$ and $\langle \cdot \rangle$ denotes average over input data distribution. In the teacher student setup the weight of the teacher network ($f^*$) is fixed beforehand. Hence, from now onward we will denote the GE as a function of the student network, i.e., as $\epsilon(f)$.  ~\cite{goldt2019dynamics} showed that GE $\epsilon(f)$ (expected error on the unseen data, for details see S31 of \cite{goldt2019dynamics}) for the student network is a function of the following \textit{order parameters},
\begin{equation}\label{order-parameters}
Q_{ik} = \frac{w_i^T w_k}{N}, \;\;\; R_{in} = \frac{w_i^T w_n^*}{N},\;\;\; R_{mn} = \frac{w_m^{*T} w_n^*}{N}.
\end{equation}
Intuitively, these order parameters measure the similarities between and within the hidden nodes of teacher and student networks. Our theoretical results assume~\cite{goldt2019dynamics}:
\begin{enumerate}\label{assumptions}
\setlength{\itemsep}{0.0pt}
    \item[(A1)] If $\bm{x}= (x_1,\ldots,x_N)$ is an input then $x_i \in \mathcal{N}(0,1).$ Also, $N \rightarrow \infty.$
    \item[(A2)] Both the teacher and the student networks have only one hidden layer.
    \item[(A3)] $K\geq M$ and $K= Z \cdot M$ where $Z \in \mathbb{Z}^{+}.$
    \item[(A4)] The activation in the hidden layer is sigmoidal for both teacher and student network.
    \item[(A5)] The output $\in$ $\mathbb{R}$ (i.e., regression problem).
    \item[(A6)] The order parameters (see section \ref{teach-stud}) satisfy the ansatz as in (S58) - (S60) of~\cite{goldt2019dynamics}. This ansatz intuitively states that the every student hidden node specializes in learning a specific teacher hidden node and for each teacher hidden node there is a student hidden node which learns that teacher node.
    \item[(A7)] No noise is added to the labels generated by the teacher network, i.e., $\sigma = 0$ in \eqref{label-teacher}.
\end{enumerate}

\section{GE of Pruned Network in Teacher-Student Setup}\label{ge-teacher-student}
We compare the performance of student networks pruned using different techniques as in Table \ref{tab:prune-methods} by analyzing their GE (see Figure \ref{fig:intuition}). For node and edge pruning comparison, we choose the parameters $k_n$ and $k_e$ (see Table \ref{tab:notations}) such that the total number of parameters of the networks remain same, i.e., they satisfy,
\begin{equation}\label{node-edge-relation}
   \frac{k_n}{K} = \Lim{N \rightarrow \infty}\frac{k_e}{N} = c,
\end{equation}
where $c \in [0,1]$ is a constant. It is important to note that since we assume that the number of student nodes is more than the number of teacher nodes, which means multiple student nodes learn the same teacher node (see Figure 3 of~\cite{goldt2019dynamics}; also in Figure \ref{fig:intuition}: two student hidden nodes learn one teacher hidden node, shown in same color). From~\cite{goldt2019dynamics}, we know that, in noiseless case ($\sigma = 0$ in \eqref{label-teacher}), the student network learns the teacher network completely when trained till convergence, i.e., the GE becomes $0$. When we prune the student network, this GE increases, which we then analyze for different types of pruning under certain assumptions (see Section \ref{assumptions} (A1)-(A7)).
\begin{table*}[t]
  \caption{Different pruning methods and notations for their GE. Here $f$ denotes the pruned student network. u.a.r. and w.p. stand for \textit{uniformly at random} and \textit{with probability} resepectively.}
  \begin{center}
  \small
\scalebox{1}{
  \begin{tabular}{|c|c|c|c|c|}
    \hline
    \textbf{Pruning Method} & \textbf{Procedure} &\textbf{Retained} &\textbf{GE without} & \textbf{GE with}\\
    & & \textbf{Parameters} &\textbf{reweighting} & \textbf{reweighting}\\
    \hline
    Random Node & Keep $k_n$ nodes u.a.r. & $k_n$ hidden nodes & $\epsilon_{k_n}^{Rand\: Node}(f)$ & $\hat{\epsilon}_{k_n}^{Rand\: Node}(f)$\\
    \hline
    Importance Node & \cite{he2014reshaping} & $k_n$ hidden nodes & $\epsilon_{k_n}^{Imp\: Node}(f)$ & $\hat{\epsilon}_{k_n}^{Imp\: Node}(f)$\\
    \hline
    DPP Node & see Section \ref{prelims} & $k_n$ hidden nodes & $\epsilon_{k_n}^{DPP\: Node}(f)$ & $\hat{\epsilon}_{k_n}^{DPP\: Node}(f)$\\
    \hline
    Random Edge & Keep an edge w.p. $c$ & $k_e$ incoming edges & $\epsilon_{k_e}^{Rand\: Edge}(f)$ & $\hat{\epsilon}_{k_e}^{Rand\: Edge}(f)$\\
     & for each hidden node & per hidden node & & \\ 
    \hline
\end{tabular}}
\end{center}
\label{tab:prune-methods}
\end{table*}

\subsection{Comparing Node Pruning Methods}
We theoretically show that the increment in GE due to \textsc{DIVNET} is less than that for random and importance node pruning methods, justifying the empirical findings of~\cite{mariet2015diversity}. The proof proceeds with the following steps: (1) Theorem \ref{dpp-np-error} provides a closed form expression of the GE after DPP node pruning. (2) Theorem \ref{dpp-betterthan-rand} shows that: (a) GE of random node pruning is greater than GE of DPP node pruning (b) GE of random node pruning with reweighting is greater than GE of \textsc{DIVNET} (c) GE of importance node pruning is greater than GE of \textsc{DIVNET}.

\begin{figure*}[t]
    \centering
    \includegraphics[width=\textwidth]{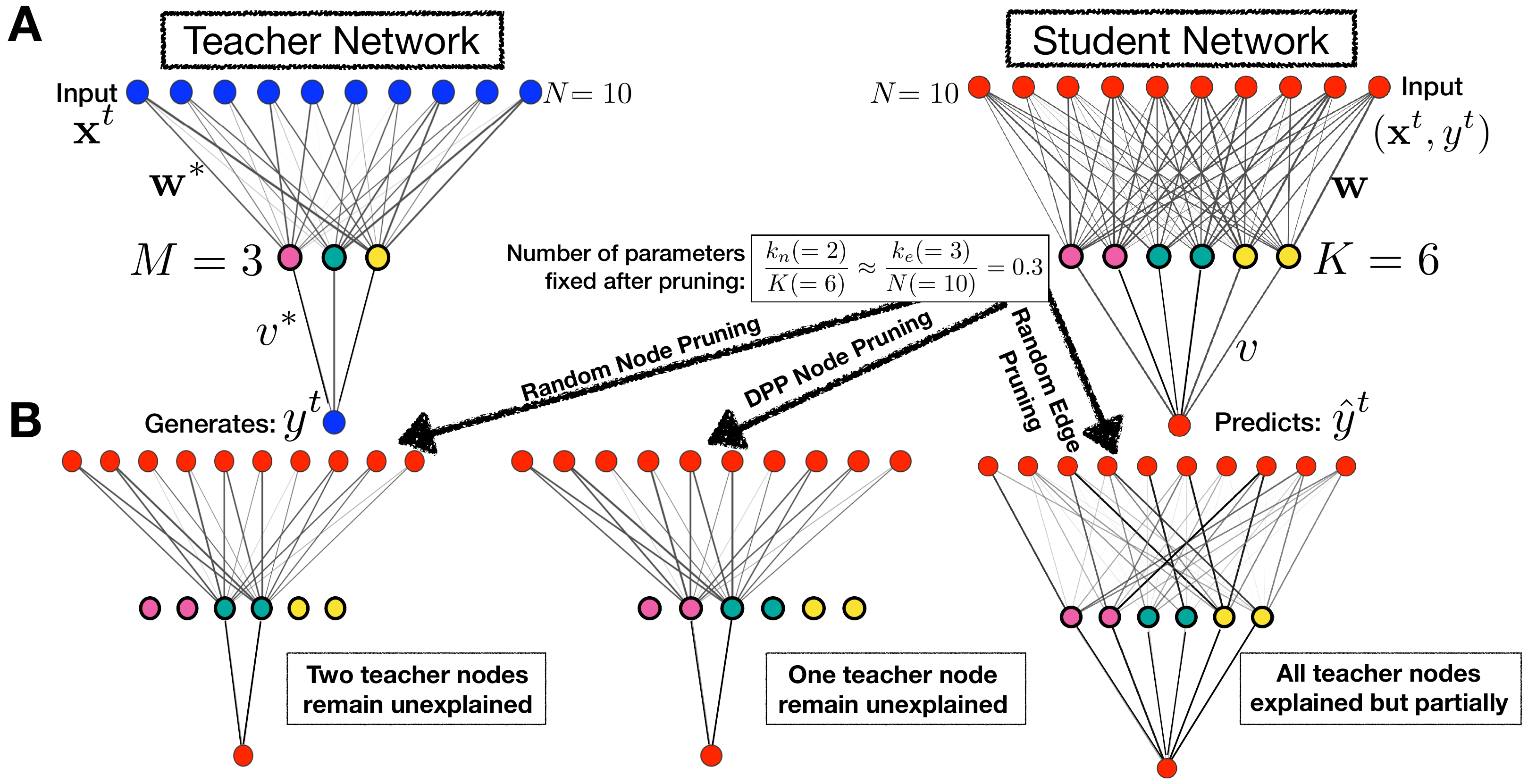}
    \caption{\textbf{(A)} Two layer teacher-student framework: A teacher neural network with 3 hidden nodes (left) and a student network with 6 hidden nodes (right). Input data (i.i.d) along with its label generated by teacher network are fed to student network to predict. \textbf{(B)} Intuitive example for 3 types of pruning on student network. For $k_n=2$, random node pruning might only be able to explain 1 teacher hidden node, whereas DPP node pruning will always retain (partial) information about 2 teacher hidden nodes, hence preforms better. Random edge pruning retains sparse information about all 3 teacher nodes which is enough to outperform DPP node pruning. All notations follow Table \ref{tab:notations}.}
    \label{fig:intuition}
\end{figure*}

\begin{theorem}\label{dpp-np-error}
Assume $(A1)-(A7).$  Let $k_n\leq M$ nodes are selected by the DPP  Node pruning method, 
\begin{equation}\label{dpp-np-closed-form}
\epsilon_{k_n}^{DPPNode}(f) = (v^*)^2 \left [ \frac{k_n}{6}\left( 1 - \frac{1}{Z}\right)^2 + \frac{M-k_n}{6} \right ]
\end{equation}
and
\begin{equation}\label{rewt-dppnp}
\hat{\epsilon}_{k_n}^{DPPNode}(f) = (M-k_n) \times \frac{(v^*)^2}{6}.
\end{equation}
\end{theorem}

\textbf{Proof Idea of Theorem \ref{dpp-np-error}:} Proof of the above theorem (details in the appendix C) is based on two factors: (1) Results from \cite{goldt2019dynamics} assure that analyzing the \textit{order parameters} is enough to obtain closed form of GE. (2) We exploit the observation that the expected kernel of the DPP node pruning is same as the order parameter $Q$ (see appendix B for proof and Figure \ref{fig:simulation-results} E) which, following \cite{goldt2019dynamics}, is a block diagonal matrix with $M$ blocks. By property of DPP, the pruning method will retain a subset of student hidden nodes with at most 1 hidden node from each block when $k_n\leq M$ (see Figure \ref{fig:simulation-results} G). 

\begin{remark}\label{dpp-node-error}
As the expected DPP kernel is block-diagonal matrix, the stochasticity in subset selection via DPP does not impact GE when subset size is fixed and it only depends on size of pruned subsets.
\end{remark}

\begin{remark}
Our theorem uses $k_n \leq M$, however, in practice the kernel may have non-zero off-diagonal entries when the assumption (A1) about input data is violated. As a result the probability of sampling a subset of size $k_n > M$ may be nonzero.
\end{remark}

\textbf{Connection to Lottery Ticket Hypothesis: }
An interesting direction of research is to find small sub-networks from an overparameterized network with comparable performance. The existence of such networks is hypothesized in Lottery Ticket hypothesis \cite{frankle2018lottery}. Interestingly, recent work shows that pruning helps find such networks even without retraining \cite{ramanujan2020s,malach2020proving} and in our work we explore a sub-network in the teacher student setup.

Note that from Eq \eqref{rewt-dppnp}, when $M$ student nodes are kept after pruning, i.e., $k_n=M$, then the GE of the DPP node pruned network is $0$ which is GE of the original student network. Hence, from the fact that $K>M$ we can conclude that DPP node pruning can find out the winning ticket, i.e., a small sub-network with much less number of parameters than the original unpruned network but with same performance guarantee.  
\begin{theorem}\label{dpp-betterthan-rand}
Assume $(A1)-(A7).$ Then for $k_n\leq M$ we have,

\begin{equation}
\mathbb{E}_f\left[\epsilon_{k_n}^{Rand\:Node}(f)\right] > \epsilon_{k_n}^{DPP\:Node}(f')
\end{equation}
and
\begin{equation}
\mathbb{E}_f\left[\hat{\epsilon}_{k_n}^{Rand\:Node}(f)\right] > \hat{\epsilon}_{k_n}^{DPP\:Node}(f')
\end{equation}
and,
\begin{equation}
\epsilon_{k_n}^{Imp\:Node}(f') > \hat{\epsilon}_{k_n}^{DPP\:Node}(f'),
\end{equation}
i.e., DPP node pruning outperforms random node pruning in the above setup.
Here the expectation is taken over the the subsets of hidden nodes of size $k_n$ chosen u.a.r.
\end{theorem}
\begin{remark}
Reweighting for DPP/random node pruning follow procedure in Section 2.2 of \cite{mariet2015diversity}.
\end{remark}
\textbf{Proof Idea of Theorem \ref{dpp-betterthan-rand}:} In random and importance node pruning, two student nodes which learn the same teacher node may both survive after pruning with non-zero probability, unlike DPP node pruning (Figure \ref{fig:intuition} (B)). Hence, more teacher nodes may remain unexplained by the student network after random or importance node pruning, resulting in increased GE (details in appendix C).

Together, Theorem \ref{dpp-np-error} and \ref{dpp-betterthan-rand} gives theoretical guarantees for all empirical results of \cite{mariet2015diversity}. Theorem 1 further allows us to show that DIVNET indeed satisfies the stronger version of Lottery Ticket Hypothesis as recently explored in \cite{ramanujan2020s,orseau2020logarithmic}. Importance node pruning with reweighting may be better than \textsc{DIVNET} and was not explored in \cite{mariet2015diversity}. 

\subsection{Comparing Node and Edge Pruning Methods}\label{effect-edge}
In random edge pruning method, for each student hidden node, an incoming edge is kept with probability $c = \Lim{N \rightarrow \infty}\frac{k_e}{N}$. Majority of empirical studies throughout literature use random edge or node pruning as a baseline for empirical comparison (see papers in \cite{blalock2020state}) making it an obvious candidate for our theoretical comparisons as well. It has been shown empirically by \cite{mariet2015diversity} and theoretically by us that DPP node pruning is an above baseline node pruning method. In this section we show that baseline random edge pruning outperforms DPP node pruning which is consistent with the empirical observations that sparse models outperform dense models (section 3.2 of \cite{blalock2020state}). Specifically, here we show that GE after random edge pruning is less than GE after DPP node pruning.  Our proof proceeds as follows: (1) Theorem \ref{random-edge-error} gives a closed form expression for the GE after random edge pruning (2) Theorem \ref{random-edge-vs-DPP-node} then shows that GE of random edge pruning is less than GE of DPP node pruning.

\begin{theorem}\label{random-edge-error}
Assume $(A1)-(A7).$ Consider the random edge pruning method with parameter $\Lim{N \rightarrow \infty}\frac{k_e}{N} = c$ (here $c$ is a constant between 0 and 1). Then the GE  $\epsilon_{c}^{Rand\: Edge}\left(\mathbb{E}\left[ f\right]\right)$ is,
\begin{equation}\label{rand-edge-closed-form}
\begin{split}
  &\frac{M(v^*)^2}{\pi}\bigg[ \frac{1}{Z} \arcsin \frac{c}{1+c} +\left(1-\frac{1}{Z}\right) \arcsin \frac{c^2 }{1+c}\\
&+ \frac{\pi}{6} - 2 \arcsin \frac{c }{\sqrt{2(1+c)}}\bigg] .  
\end{split}
\end{equation}
\end{theorem}

\begin{remark}
Theorem \ref{random-edge-error} gives closed form for ``GE of the expected network" after pruning instead of the ``expected GE of the network" after pruning. However, in the thermodynamic limit ($N \rightarrow \infty$), the order parameters as in Section \ref{prelims} are highly concentrated near their expected values and the two quantities hence become equal.
\end{remark}

\begin{theorem}\label{random-edge-vs-DPP-node}
Assume $(A1)-(A7)$. Let $k_n$ and $c$ satisfy \eqref{node-edge-relation}, and $0 \leq c \leq \frac{1}{Z}$ and $Z\geq 4$. Then

\begin{equation}
\epsilon_{k_n}^{DPP\:Node}\left(f \right)  \geq  \epsilon_{c}^{Rand\:Edge}\left(\mathbb{E}\left [f\right] \right),
\end{equation}
i.e., Random edge pruning outperforms DPP node pruning in the above setup. 
\end{theorem}

\textbf{Proof Idea of Theorem \ref{random-edge-vs-DPP-node}:} When $k_n \leq M$, node pruned student network leaves at least $(M-k_n)$ teacher nodes unexplained, whereas after random edge pruning, student network can retain at least partial information about every teacher node (see Figure \ref{fig:intuition} (B)). After a pruning routine, the sum of partial information about all teacher nodes in an edge pruned student network dominates the sum of information for the explained subset of teacher nodes in a node pruned student network. 
\begin{figure*}[t]
    \centering
    \includegraphics[width=\textwidth]{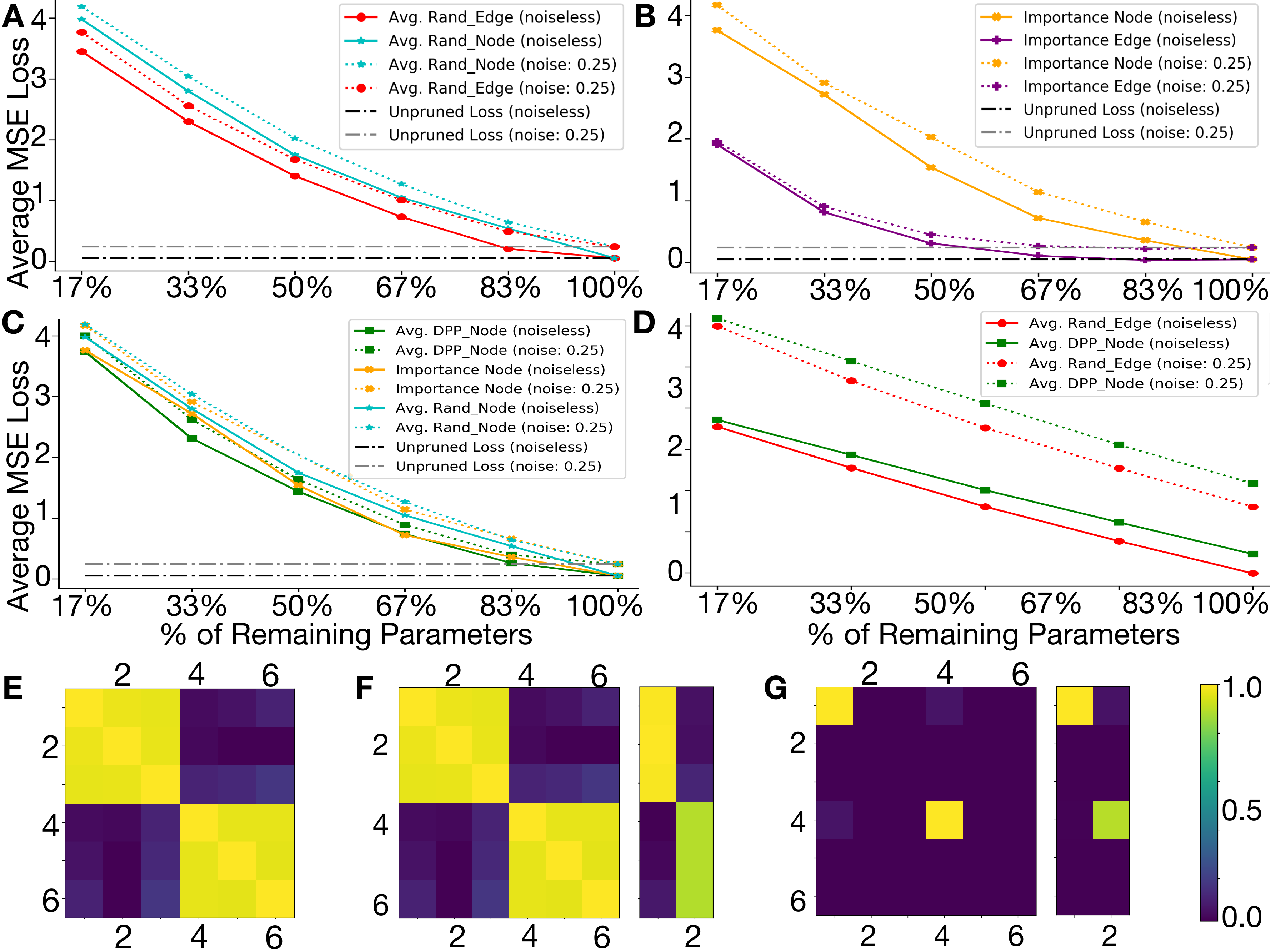}
    \caption{Simulation results in teacher student setup, $M=2$ and $K=6$ for (A-G). \textbf{(A-B)} Edge pruned networks perform better than node pruned networks in all 3 types of pruning methods (random (A), importance (B)), validating Conjecture \ref{sparse-dense}. \textbf{(C)} DPP Node pruning performs better than importance and random node pruning (Theorem \ref{dpp-betterthan-rand}) \textbf{(D)} Baseline random edge pruning beats DPP node pruning (Theorem \ref{random-edge-vs-DPP-node}). For (D), $M=5$ and $K=20.$ \textbf{(E)} The kernel of DPP node pruning is same as $Q$  \textbf{(F)} Order parameters, $Q$ (same as (E)) and $R$ of the  unpruned student network. \textbf{(G)} When only keeping 2 nodes, DPP node pruned student network keeps one from each block shown in (G). } 
    \label{fig:simulation-results}
\end{figure*}

\textbf{Observations:} From Theorem \ref{dpp-betterthan-rand} and \ref{random-edge-vs-DPP-node}, we conclude that random edge pruning outperforms random node pruning. Further, using Theorem \ref{dpp-betterthan-rand} and the intuition that importance edge pruning is better than random edge pruning, we expect that importance edge pruning will outperform importance node pruning. Figure \ref{fig:simulation-results} confirms this empirically in the teacher student setup. These observations leads to the conjecture that for a fixed pruning method, edge pruning outperforms node pruning.

\begin{conjecture}\label{sparse-dense}
Assume $(A1)-(A7)$. Let $k_n$ and $c$ satisfy \eqref{node-edge-relation} and $Prune$ denotes a fixed pruning method (e.g. Rand, Imp) which can be applied to both node and edge. Then, $ \exists  c_{\epsilon} \in (0,1]$ such that for $0 \leq c \leq c_{\epsilon}$,
\begin{equation}
\epsilon_{k_n}^{Prune\:Node}\left(f \right)  \geq  \epsilon_{c}^{Prune\:Edge}\left(f \right).
\end{equation}
\end{conjecture}

Together, Theorem \ref{random-edge-error}, \ref{random-edge-vs-DPP-node} and Conjecture \ref{sparse-dense} are consistent with empirical observations of ~\cite{blalock2020state}: sparse networks after edge pruning perform better on the unseen test data than dense networks after node pruning with fixed number of parameters. To the best of our knowledge, ~\cite{blalock2020state} based their claims from empirical observations of pruning studies in which the pruned networks were not reweighted. This motivated our choice of comparing GE for DPP node pruning and random edge pruning without any reweighting. However, with reweighting from \cite{mariet2015diversity}, GE of DPP node pruning will be less than GE of random edge pruning, highlighting the impact of reweighting proposed by \cite{mariet2015diversity} (proof and details in appendix C).

We find that GE analysis on teacher-student setup is flexible for various pruning methods and this framework can be extended to theoretically understand other pruning methods which are outside the scope of this work. 

\begin{figure*}[t]
    \centering
    \includegraphics[width=\textwidth]{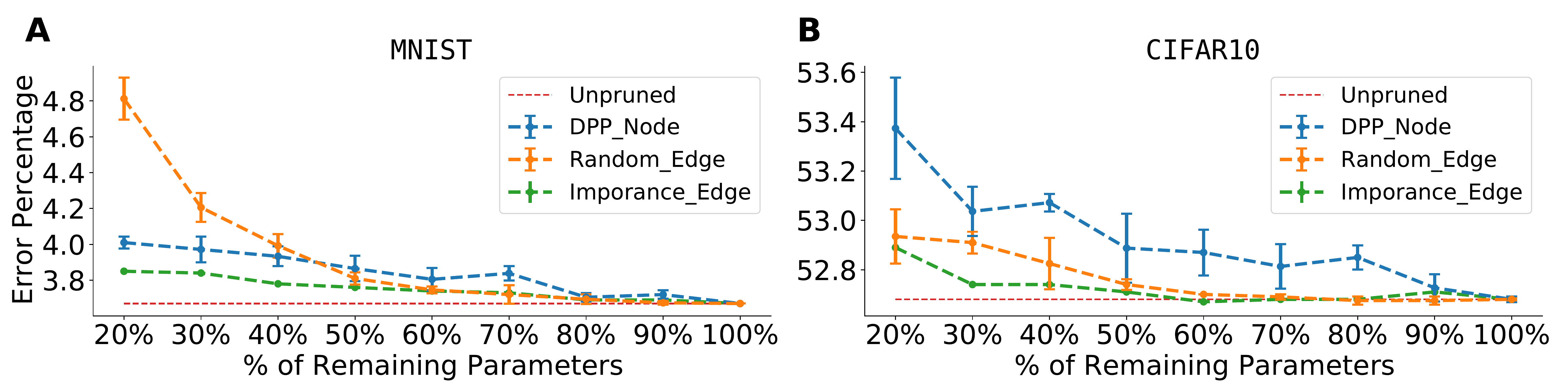}
    \caption{ Comparing different edge pruning methods with DPP Node pruning method on the \texttt{MNIST} (A) and \texttt{CIFAR10} (B) dataset. Horizontal axis represents the percentage of remaining parameters in $1^{st}$ layer after pruning. The vertical axis shows corresponding test error. Both magnitude based edge pruning method (importance pruning) and baseline random edge pruning method outperforms DPP Node pruning which confirms Theorem \ref{random-edge-vs-DPP-node} and the conjecture proposed in \cite{blalock2020state}.} 
    \label{fig:data-results}
\end{figure*}

\section{Experiments} 
\subsection{Simulations}\label{Simulation}
We run the DPP node, random edge/node, and importance edge/node pruning simulations under the teacher-student setup. For all the simulations, we sampled the $800000$ i.i.d input samples from $\mathcal{N}(0,1)$ as training data and $80000$ as testing data. Following notations from Table \ref{tab:notations}, we set $M = 2$, $K = 6$, $N = 500$, and $v^* = 4$. The first layer teacher network weights $\bm{w}^*$ and all the student network parameters $\theta = \{\bm{w},\bm{v}\}$ were drawn independently from $\mathcal{N}(0,1)$ as initialization. We choose learning rate $\eta = 0.50$, and it is scaled to $\frac{\eta}{\sqrt{N}}$ for $\bm{w}$ and $\frac{\eta}{N}$ for $\bm{v}$. We run the simulations for both noiseless ($\sigma = 0$ in \eqref{label-teacher}) and noisy ($\sigma =0.25$) output labels. For comparisons between node and edge pruning, we use the node-to-edge ratio [$1:83, 2:166, 3:250, 4:333, 5:417, 6:500$] to keep the number of parameters the same, given $N = 500$, $K = 6$, and $M = 2$. In addition, we run the same simulation with $K = 5$ and $M = 20$, see Figure~\ref{fig:simulation-results}D. For other simulation details and results, see appendix. Note that no pruning method undergoes reweighting for reported simulation results which we therefore use to verify and validate our theoretical results without reweighting.

\textbf{Key Observations:} 
\begin{itemize}
    \item The expected kernel of the DPP node pruning and the $Q$ matrix are the same which we exploit for Theorem \ref{dpp-node-error} (Figure \ref{fig:simulation-results}E,F).
    \item For $k_n =2$ and $M=2$, DPP node pruning chooses exactly one node from each of the diagonal block of the kernel (see Figure \ref{fig:simulation-results}G) which validates Theorem \ref{dpp-node-error}.
    \item DPP node pruning outperforms random and importance node in both noisy and noiseless case (see Figure \ref{fig:simulation-results}C) which confirms Theorem \ref{dpp-betterthan-rand}. 
    \item Random edge pruning is better than DPP node pruning for $c\leq \frac{1}{Z}$ with $Z=4$ and $M=5$ in both noisy and noiseless cases (see Figure \ref{fig:simulation-results}D), validating Theorem \ref{random-edge-vs-DPP-node}. 
    \item We see Conjecture \ref{sparse-dense} holds for random, importance edge and node pruning (see Figure \ref{fig:simulation-results}A,B)
\end{itemize}

\subsection{Real Data}\label{real-data-experiments}

In this section, we compare \textsc{DIVNET} by \cite{mariet2015diversity} with random edge pruning with reweighting, and importance edge pruning with reweighting on the \texttt{MNIST}~\cite{lecun2010mnist} and \texttt{CIFAR10}~\cite{krizhevsky2009learning} datasets. We used the exact same network architectures as in Table 1 of \cite{mariet2015diversity} for \texttt{MNIST} and \texttt{CIFAR10}, respectively. Note that, for the real data we consider network structures with multiple layers. Following \cite{mariet2015diversity}, we performed all pruning methods on the first layer. 
We compare the number of parameters as $k_e = \frac{k_n(d_{\text{input}} + h_2)}{h_1} - h_2$ where $k_e$ is the number of edges kept for each node in edge pruning, and $k_n$ is the number of nodes kept in the hidden layer for node pruning; $d_{\textrm{input}}$, $h_1$, and $h_2$ represent the dimension of the input, size of the first hidden layer, and size of the second hidden layer, respectively. As in \cite{mariet2015diversity}, $h_1 = h_2$. We trained our model until the training error reaches predefined thresholds (Table 1 in \cite{mariet2015diversity}) and then perform the pruning. For hyperparameters and other details, see~\ref{hyperparam_real_data}.

\textbf{Remark:} Note that we have not presented the results comparing different node pruning methods among themselves as they were already discussed in \cite{mariet2015diversity}.
\\

\textbf{Key observations:} 
\begin{itemize}
    \item Baseline random edge pruning method outperforms DIVNET across all percentages of parameters retained in the network for \texttt{CIFAR10} dataset shown in Fig \ref{fig:data-results} B. However, for \texttt{MNIST} dataset, DIVNET performs better than random edge intitally but if $> 40\%$ of parameters are retained in the network random edge outperforms DIVNET (see Fig \ref{fig:data-results} A).
    \item Importance edge pruning performs better than both DIVNET and the baseline random edge pruning method on both the real data sets which highlighting the potential of magnitude based pruning method (see Fig \ref{fig:data-results} A and B).
\end{itemize}

\section{Discussion and Future Work}
Our work takes the first step to develop theoretical comparison for empirical observations of pruning methods in feed forward neural networks. We identify the usefulness of teacher-student setup for providing theoretical guarantees of pruning methods. We then use this setup to theoretically show that DIVNET should indeed outperform random and importance node pruning techniques. We further show that random edge pruning outperforms DPP node pruning providing a theoretical proof for the popular empirical observation: sparse (node) networks perform better than dense (edge) pruned networks for fixed number of parameters. Finally, we also are able to show that DIVNET satisfies a stronger version of the Lottery Ticket Hypothesis. Our work consolidates the understanding of a particular class of node and edge pruning theoretically. 

When comparing two neural networks, using the number of parameters may not always be the optimal choice, instead, measuring the \textit{capacity} and \textit{expressiveness} of neural networks~\cite{arora2017generalization} can provide new insights. All our theoretical results have been proved on single hidden layer neural networks which gives future scope of extending them to multiple hidden layer networks. However, our empirical results hold for neural networks with multiple hidden layers suggesting the possibility of generalization of our results.

Throughout this work, we focus only on pruning methods in which a feedforward pre-trained neural network is pruned once without retraining.  We choose this class for two primary reasons: (1) it is feasible to make theoretical comparisons with closed form solutions of GE, and, (2) with some assumptions, it has been shown by recent studies \cite{ramanujan2020s,malach2020proving,orseau2020logarithmic} that every sufficiently over-parameterized network contains a sub network that, even without training, achieves comparable accuracy to the trained large network. This proven conjecture is even stronger than the Lottery Ticket Hypothesis \cite{frankle2018lottery}. Hence, comparing performance of pruning methods within the aforementioned class in the teacher-student setup allowed us explore the existence of such a sub network. 

We compare our theoretical results with random pruning and importance pruning which subsumes ideas underlying vast majority of pruning techniques and do not focus on any specific algorithm. A more specific algorithm based justification can also be an extension (may not always be trivial however) of this paradigm. 

We introduce the teacher-student setup for proving results related to pruning methods which can further be extended to prove other empirical results in the pruning domain. Such theoretical insights can also be used as a means to guide development of theory-motivated new and better pruning algorithms on other neural network architectures like CNNs and RNNs in future work.

\bibliography{uai2021-template}

\appendix
\onecolumn 

\section{GE in Two Layer Network}
For the theoretical analysis we consider the following assumptions from~\cite{goldt2019dynamics}
\begin{enumerate}
\setlength{\itemsep}{0.0pt}
    \item[(A1)] If $\bm{x}= (x_1,\ldots,x_N)$ is an input then $x_i \in \mathcal{N}(0,1).$ Also, $N \rightarrow \infty.$
    \item[(A2)] Both the teacher and the student networks have only one hidden layer.
    \item[(A3)] $M,K$ denotes the number of hidden nodes for the teacher and student network respectively and $K\geq M$ and $K= Z \cdot M$ where $Z \in \mathbb{Z}^{+}.$
    \item[(A4)] The activation in the hidden layer is sigmoidal for both teacher and student network.
    \item[(A5)] The output $\in$ $\mathbb{R}$ (i.e., regression problem).
    \item[(A6)] The order parameters satisfy the ansatz as in (S58) - (S60) of~\cite{goldt2019dynamics}.
    \item[(A7)] No noise is added to the labels generated by the teacher network, i.e., $\sigma = 0$.
\end{enumerate}{}
With the above assumptions, authors of ~\cite{goldt2019dynamics} gave a closed form of the GE as follows:
\begin{equation}\label{ge-breakdown}
    \epsilon_{g}  =  f_1(Q) + f_2(T) - f_3(R,Q,T)
\end{equation}
where,
\begin{align}
    f_1(Q) & = \frac{1}{\pi} \sum_{i,k} v_i v_k \arcsin \frac{Q_{ik}}{\sqrt{1+Q_{ii}} \sqrt{1+Q_{kk}}}\label{f1}\\ 
    f_2(T) & = \frac{1}{\pi}  \sum_{n,m} v_n^* v_m^* \arcsin \frac{T_{nm}}{\sqrt{1+T_{nn}} \sqrt{1+T_{mm}}}\label{f2}\\
    f_3(R,Q,T) & = \frac{2}{\pi}  \sum_{i,n} v_i v_n^* \arcsin \frac{R_{in}}{\sqrt{1+Q_{ii}} \sqrt{1+T_{nn}}}\label{f3}
\end{align}
where $Q,R,T$ are the order parameters as defined in main text. We also have the assumption \eqref{node-edge-relation} about the relation between number of edges and nodes kept after pruning.

\section{Properties of DPP Kernel}\label{kernel-prop}
In main text we see that each node in the hidden layer of a student network carries certain amount of information about the training data and it is captured in a vector form. We create an information matrix by accumulating the information vectors of these hidden nodes.  For simplicity of theoretical analysis, we have considered the kernel as the inner product of the information matrix. In the thermodynamic limit, the inner product is divided by the input dimension. Formally, if $\bm{h}_i$ and $\bm{h}_j$ are the information at $i^{th}$ hidden node and $j^{th}$ hidden node respectively, then
$$L_{ij} =  \frac{1}{N}  \frac{1}{n} \bm{h}_i^T \bm{h}_j$$
where $n$ is the total number of training examples. It can be seen that the analysis for the kernel defined in main text is similar. Note that all analyses are for the student network trying learn from the teacher network. Refer to main text for details of notations.

\begin{lemma}\label{expected-node}
Assume (A1) - (A7). Then the expected kernel of DPP Node for the hidden layer is the order parameter $Q$.
\end{lemma}

\begin{proof}[Proof of Lemma \ref{expected-node}]
For the two-layer teacher-student setup, the hidden layer gets information $(\bm{h}_1,\ldots,\bm{h}_{K})$ from the input layer, where $\bm{h}_i= (h_{i1},\ldots,h_{in})$
and $h_{ij} (= t_j^T\bm{w}_i)$ is the information at $i^{th}$ hidden node on $j^{th}$ input data ($t_j$). Hence, 
$$\bm{h}_i^T\bm{h}_j  =  \sum_{k=1}^n h_{ik} h_{jk} =  \sum_{k=1}^n  t_k^T\bm{w}_i \cdot t_k^T\bm{w}_j  = \sum_{k=1}^n  \bm{w}_i^T t_k \cdot t_k^T\bm{w}_j  = \sum_{k=1}^n  \bm{w}_i^T (t_k  t_k^T) \bm{w}_j$$
But for the given input distribution (i.i.d. Gaussian), $\mathbb{E}[t_k  t_k^T] = \bm{I}_{N\times N}.$ Hence, $\Lim{N \rightarrow \infty} \mathbb{E}[L_{ij} ]  =\Lim{N \rightarrow \infty}\mathbb{E}[ \frac{1}{N}\frac{1}{n}\bm{h}_i^T\bm{h}_j]= \Lim{N \rightarrow \infty} \frac{1}{N} \bm{w}_i^T \bm{w}_j=Q_{ij}$, and we have the lemma.
\end{proof}

From \cite{goldt2019dynamics} we know that $Q$ is a block diagonal matrix where each ``block" (or ``group" used interchangeably henceforth) refers to the set of student hidden nodes that represent (explian/learn) one particular teacher hidden node.

% \begin{lemma}\label{expected-edge}
% Assume (A1) - (A7). Then the expected kernel of DPP Edge for the $s^{th}$ hidden node is \[
%   \mathbb{E}[L_{ij}^s] =
%   \begin{cases}
%     \frac{1}{N} w_{is}^2  & \text{ if } i=j \\
%     0 & \text{ otherwise } 
%   \end{cases}
% \]
% \end{lemma}
% \begin{proof}[Proof of Lemma \ref{expected-edge}]
% For the $s^{th}$ hidden node, we choose a subset of incoming edges w.r.t. a DPP. The representation for the $j^{th}$ incoming edge is  $\bm{u}_j^s = (u_{1j}^s,\ldots,u_{nj}^s),$ where $u_{kj}^s=t_{kj}w_{js}.$ Hence 
% $$L_{ij}^s = \frac{1}{N} \frac{1}{n} \sum_{k=1}^n u_{ki}^s u_{kj}^s = \frac{1}{N} \frac{1}{n} \sum_{k=1}^n t_{ki}w_{is} t_{kj}w_{js}$$
% As the input is drawn from i.i.d. Gaussian distribution, it can be seen that $\mathbb{E}[t_{ki} t_{kj}] = 1 \iff i=j$. Hence we have 
% \[
%   \mathbb{E}[L_{ij}^s] =
%   \begin{cases}
%     \frac{1}{N} w_{is}^2  & \text{ if } i=j \\
%     0 & \text{ otherwise } 
%   \end{cases}
% \]
% And we have the lemma.
% \end{proof}

% This lemma basically states that the DPP edge method under the teacher student setup has a diagonal kernel.

\section{Proof of the Theorems}\label{proof-of-theorems}

\begin{proof}[Proof of Theorem \ref{dpp-np-error} and \ref{dpp-betterthan-rand}]
Let $H_R = \left\{ h_{i_1}, \ldots, h_{i_{k_n}} \right\}$ be the set of selected nodes by DPP Node pruning method. Recall from~\cite{goldt2019dynamics} that every student hidden node specializes in learning a teacher node. Denote $t(h)$ to be the  teacher node learnt by $h$. $S_{m}\subseteq H_R$  be the set of selected hidden nodes of the pruned network which learnt the $m^{th}$ teacher node , i.e., $S_m = \{h \in H_R | t(h) = t_m \}$ ($t_m$ is the $m^{th}$ teacher node). Hence, $prn = |\{\mathds{1}(|S_m|>0)| 1 \leq m \leq M \}| $ is the  number of teacher nodes explained by the pruned network and W.L.O.G. we can assume that $t_1, \ldots, t_{prn}$ are those set of teacher nodes. Let $l_1, \ldots, l_{prn}$ be the number of student nodes in the pruned network which learn the corresponding teacher node. Note that, $\sum_{i=1}^{prn} l_i =  k_n$ and $l_i\leq Z$ (where $Z$ is the number of student nodes dedicated to learn a single teacher node in the unpruned network) for all $i$. Applying Lemma \ref{np-error} directly we can see that the GE for the pruned network  is 
\begin{equation}\label{randnperr}
    \frac{(v^*)^2}{6}\left[ \sum_{i=1}^{prn}\left(1-\frac{l_i}{Z} \right)^2\right] + \frac{(M-prn)(v^*)^2}{6}
\end{equation}

The first part of \eqref{randnperr} is the GE for the group whose corresponding teacher node is partially explained and the second part accounts for the GE due to unexplained teacher nodes (number of such teacher nodes are $M-prn$).  From Lemma \ref{expected-node} we know that the expected kernel matrix for DPP Node pruning is the order parameter $Q$ and it becomes a block diagonal matrix after the training converges, where size of each block is $Z$ (which is also the number of student nodes dedicated to learn a single teacher node in the unpruned network). Because of the block diagonal property of the DPP kernel matrix, at most 1 student hidden node will be chosen from each block, i.e., $l_i =1$ $\forall i$. Hence, $prn = k_n$. From Lemma \ref{np-error} we can see that the GE of node pruned network only depends on the number of student node survived in each block after pruning, and, for DPP node pruning, it is always 1 (given $k_n \leq M$). This is why there is no expectation in the GE term. So for DPP node pruning the GE is,
$$\epsilon_{k_n}^{DPP\: Node}(f) = (v^*)^2 \left [ \frac{k_n}{6}\left( 1 - \frac{1}{Z}\right)^2 + \frac{M-k_n}{6} \right ].$$
Each of the $k_n$ student nodes in the pruned network learns a different teacher node. Consider one such teacher node and call it $t_i$. In the unpruned network, there are $Z$ student hidden nodes which learn a single teacher node $t_i$, only one of which survives after DPP node pruning. The first part of the error is due to the removal of student nodes ($Z-1$ student nodes for each $t_i$). However, these errors can be retrieved by reweighting the survived student node. On the contrary, there are $M-k_n$  teacher nodes which don't have any representative (some student hidden node from the set of student nodes which specialized in this particular teacher node) in the pruned network. And the error (second part of the GE) due to those  nodes can not be retrieved even after reweighting. Hence, the GE after reweighting becomes,
$$(M-k_n) \times \frac{(v^*)^2 }{6}$$
Thus, we have the Theorem \ref{dpp-np-error}. 

Next, we will prove Theorem \ref{dpp-betterthan-rand}. We will show, for any network pruned by Random Node, the GE is more than the expected GE of DPP Node pruning. Recall the randomly pruned network $f$ discussed in the beginning of the proof. From Lemma \ref{np-error} we can see that for node pruning the GE only depends on the number of nodes survived in each block. From \eqref{randnperr} we have,
\begin{align}
&\;\;\;\;\;\; \epsilon_{k_n}^{Rand\: Node}(f) \nonumber  \\
& = \frac{(v^*)^2}{6}\left[ \sum_{i=1}^{prn}\left(1-\frac{l_i}{Z} \right)^2\right] + \frac{(M-prn)(v^*)^2}{6}\nonumber \\
& = \frac{(M-k_n)(v^*)^2}{6} +\sum_{i=1}^{prn} \left [ (l_i-1)\frac{(v^*)^2}{6}  +\frac{(v^*)^2}{6} \left(1 - \frac{l_i}{Z}\right)^2 \right] \nonumber \\
    & \geq \frac{(M-k_n)(v^*)^2}{6} + \sum_{i=1}^{prn} l_i \frac{(v^*)^2}{6} \left(1-\frac{1}{Z}\right)^2\label{intermediate} \\
    & = \frac{(M-k_n)(v^*)^2}{6} + l \frac{(v^*)^2}{6} \left(1-\frac{1}{Z}\right)^2\nonumber \\
    & = \epsilon_{k_n}^{DPP\: Node}(f) \nonumber
\end{align}{}
where \eqref{intermediate} follows from the inequality below:
$$(l_i-1)\frac{(v^*)^2}{6}  +\frac{(v^*)^2}{6} \left(1 - \frac{l_i}{Z}\right)^2   = l_i \frac{(v^*)^2}{6} \left[ 1 +\frac{1}{Z^2}-\frac{2}{Z}\right] \geq l_i \frac{(v^*)^2}{6} \left(1-\frac{1}{Z}\right)^2$$
which proves the first part of Theorem \ref{dpp-betterthan-rand}. The proof for the reweighted network is similar.

In case of importance node pruning, the nodes with lowest absolute value of outgoing edges are dropped. Following \cite{goldt2019dynamics} the outgoing weights of all the hidden teacher nodes are equal (we call it $v^*$).  Also, from Lemma \ref{converge-second} we see that the sum of the weights of the outgoing edges of the student nodes which learn the same teacher node add up to the outgoing edge weight of the corresponding teacher hidden node. Moreover, we assume the ansatz $v_i = v_j$ when $i,j \in G_n$, where $G_n$ denotes the set of student nodes which learn the same teacher node $t_n$. Hence, we can see that all the outgoing edges are approximately similar. We also verify this fact experimentally. Therefore, this defines an approximately uniform distribution on the set of hidden nodes. Hence, this is almost same as random node pruning and so the result follows from Theorem \ref{dpp-betterthan-rand}.  
\end{proof}

\begin{remark}
The comparison between performance of importance node pruning and \textit{DIVNET} depends on the fact that all the outgoing edges of the teacher hidden nodes are equal. However, when the outgoing weights are not equal the importance pruning first selects student hidden nodes from a group whose corresponding teacher node has the highest weight. Once all the student nodes are selected from that group then it selects the group whose corresponding teacher node has second highest outgoing edge weight and the process continues. Because of this approach, even without reweighting a complete information about the teacher node  is preserved in the pruned network.  However, in  DPP node pruning one candidate from each group (representing a particular teacher node) is selected first. But if a member is selected from a group then the reweighting method can recover the complete lost information for the corresponding group.  Hence, \textit{DIVNET} is able to preserve information about more number of teacher hidden nodes than importance pruning which results in better performance.
\end{remark}

\begin{proof}[Proof of Theorem \ref{random-edge-error}]
In this theorem, we will give the GE of the expected network pruned by the Random Edge method. Pruning is performed on the edges between input layer and the hidden layer. Hence, the order parameter changes.  From Lemma \ref{rand_edge-pruning}, we have the order parameters of the expected network (call these $Q',R',T'$). However, the weights of the second layer remain unchanged. Putting these values in \eqref{f1}, \eqref{f2} and \eqref{f3} we have, 
\begin{align}
    f_1(Q') & = \frac{1}{\pi} \sum_{i,k} v_i v_k \arcsin \frac{Q'_{ik}}{\sqrt{1+Q'_{ii}} \sqrt{1+Q'_{kk}}}\nonumber\\
    & = \frac{M(v^*)^2}{\pi} \arcsin \frac{c^2 }{1+c}  +\frac{M(v^*)^2}{Z\pi}\left[\arcsin \frac{c}{1+c}-\arcsin \frac{c^2 }{1+c}\right]
\end{align}
and,
\begin{align}
    f_3(R',Q',T') & =  \frac{2}{\pi}  \sum_{i,n} v_i v_n^* \arcsin \frac{R'_{in}}{\sqrt{1+Q'_{ii}} \sqrt{1+T'_{nn}}}\nonumber\\
    & = \frac{2 M (v^*)^2}{\pi} \arcsin \frac{c }{\sqrt{2(1+c)}}
\end{align}
Therefore, the GE of the expected network after Random Edge pruning is,
$$\frac{M(v^*)^2}{\pi}\left[ \arcsin \frac{c^2 }{1+c} + \frac{\pi}{6} - 2 \arcsin \frac{c }{\sqrt{2(1+c)}} \right] +\frac{M(v^*)^2}{Z\pi}\left[\arcsin \frac{c}{1+c}-\arcsin \frac{c^2 }{1+c}\right]$$
This proves the first part of the theorem.
\end{proof}{}

\begin{figure*}
    \centering
    \includegraphics[scale=0.35]{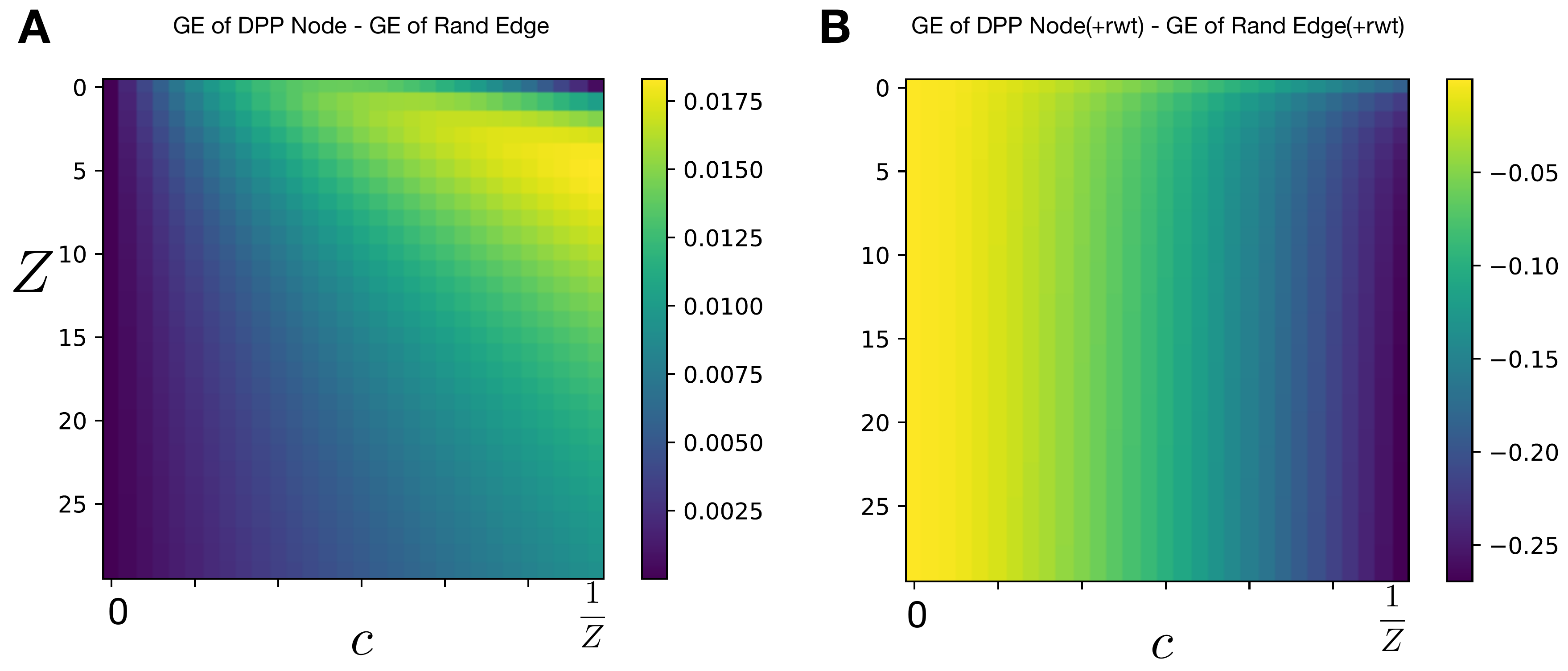}
2    \caption{ \textbf{(A)} Difference between the GE of DPP node pruning and Random edge pruning for $4 \geq Z \geq 30$. The matrix consist of only nonzero entries which proves that random edge pruning performs better than DPP node pruning when parameter count is same. \textbf{(B)} Difference between the GE of DPP node pruning with reweighting and Random edge pruning with reweighting for $4 \geq Z \geq 30$. The matrix consist of only negative entries which proves that random edge pruning can never perform better than DPP node pruning when reweighting is applied in the second layer.}
    \label{fig:DPP-np-minus-rand-edge}
\end{figure*}

\begin{proof}[Proof of Theorem \ref{random-edge-vs-DPP-node}]
Theorem \ref{dpp-np-error} and \ref{random-edge-error} provide the closed form of the GE after DPP node pruning and random node pruning respectively. Using this closed form we plot $\epsilon_{k_n}^{DPP\: Node}(f) - \epsilon_{c}^{Rand\: edge}(f)$ in Figure \ref{fig:DPP-np-minus-rand-edge} A. Here $k_n$ and $c$ satisfy \eqref{node-edge-relation}, i.e., parameter count is same after two kinds of pruning.  We can see for $Z\geq 4$ this value is $\geq 0$ given $0\leq c \leq 1.0/Z,$ which proves the theorem.
\end{proof}

\begin{remark}
Our results hold for $Z\geq 4$, where $Z$ is the number of student nodes which learn the same teacher node. This is because in DPP node pruning at most $1$ student node survives per group. As a result for larger $Z$ the lost information per group is higher (in the scale of $\left(1-\frac{1}{Z}\right)^2$).
\end{remark}

Next we state the impossibility result as discussed inmain text. We will show that, no reweighting scheme in the second layer for random edge pruning which is based on scaling can beat DPP node pruning after reweighting. Formally we have the following:

\begin{theorem}\label{reweighted-random-edge}
Assume $(A1)-(A7)$. Let $k_n$ and $c$ satisfy \eqref{node-edge-relation}, and $0 \leq c \leq \frac{1}{Z}$ and $Z\geq 4$. Assume the reweighting scheme for random edge in second layer such that, $\hat{v}_i = A v_i$.  Then $\forall A \in \mathbb{R}$ we have,

\begin{equation}
\hat{\epsilon}_{k_n}^{DPP\:Node}\left(f \right)  \leq  \hat{\epsilon}_{c}^{Rand\:Edge}\left(\mathbb{E}\left [f\right] \right)
\end{equation}
\end{theorem}

\begin{proof}[Proof of Theorem \ref{reweighted-random-edge}]
From Theorem \ref{dpp-np-error} we know that the GE after rewighting the DPP node pruned network is
\begin{equation}\label{dpp-np-rwt}
   \frac{(v^*)^2}{6} \left(M-k_n\right) = \frac{M (v^*)^2}{6}  \left(1-Zc\right)
\end{equation}
where $c$ satisfies \eqref{node-edge-relation}. Now for the given reweighting scheme in the hypothesis the GE for random edge pruning will be, 
\begin{equation}\label{rwt-rand-edge}
\frac{M(v^*)^2}{\pi}\left[ A^2 \left(\frac{1}{Z} \arcsin \frac{c}{1+c} +\left(1-\frac{1}{Z}\right) \arcsin \frac{c^2 }{1+c} \right) + \frac{\pi}{6} - 2 A \arcsin \frac{c }{\sqrt{2(1+c)}} \right] 
\end{equation}
\eqref{rwt-rand-edge} can be viewed as a quadratic equation of $A$ whose minimum correspond to the best reweighting scheme in the scaling family. In Figure \ref{fig:DPP-np-minus-rand-edge} B we compare this minimum with \eqref{dpp-np-rwt}. Formally we plotted $\hat{\epsilon}_{k_n}^{DPP\:Node}\left(f \right)  - \hat{\epsilon}_{c}^{Rand\:Edge}\left(\mathbb{E}\left [f\right] \right).$ It can be seen that this value is $-ve$ for all $0 \leq c \leq \frac{1}{Z}$, which implies GE of reweighted DPP node pruned network is always lower than reweighted random edge pruned network. 
\end{proof}

\section{Proof of Lemmas}

\begin{lemma}\label{np-error}
Assume (A1)-(A7). Let $t_1, \ldots, t_{M}$ denote the teacher hidden nodes and $l_1, \ldots, l_{M}$ denote the number of student hidden nodes in a node pruned network which learnt the corresponding teacher node. If $\sum_{m=1}^M l_m \leq M,$ then the GE of this node pruned network is,
$$\frac{(v^*)^2}{6}\left[ \sum_{m=1}^{M}\left(1-\frac{l_m}{Z} \right)^2\right]. $$
\end{lemma}

\begin{proof}
Let $G_1,\ldots,G_M$ be the subsets of student nodes such that all student nodes in $G_m$ learn the $m^{th}$ teacher node. From the assumption we have, $|G_m| = Z$ for all $m$. After pruning, a subset $P_m \subseteq G_m$ is chosen, and $|P_m| = l_m$. Denote the order parameters of the pruned network as $Q',R',T'$. For node pruning we can see that

\[
Q'_{ik} =
\begin{cases}
     Q_{ik}  & \text{ if } \exists m \text{ s.t. } h_i \in P_m \text{ and } h_k \in P_m \\
    0 & \text{ otherwise } 
  \end{cases}
\] 
Also, for the unpruned network we have 
\[
Q_{ik} =
\begin{cases}
     1  & \text{ if } \exists m \text{ s.t. } h_i \in G_m \text{ and } h_k \in G_m \\
    0 & \text{ otherwise } 
  \end{cases}
\] 
Now from \eqref{ge-breakdown} we can break down the GE into three parts. From \eqref{f1}, \eqref{f2} and \eqref{f3} we have,.  
\begin{align}
    f_1(Q') & = \frac{1}{\pi} \sum_{i,k} v_i v_k \arcsin \frac{Q'_{ik}}{\sqrt{1+Q'_{ii}} \sqrt{1+Q'_{kk}}},\nonumber\\
     & = \frac{1}{\pi} \sum_{n=1}^{M} \sum_{i,k \in P_n} v_i v_k \arcsin \frac{1}{2},\label{arcsinineq}\\
    & = \frac{1}{\pi} \sum_{n=1}^{M} \sum_{i,k \in P_n} v_i v_k \frac{\pi}{6},\nonumber\\
    & =  \frac{1}{6} \sum_{n=1}^{M} \left(\sum_{i \in P_{n}} v_i\right)^2,\nonumber\\
    & = \frac{(v^*)^2}{6} \sum_{n=1}^{M} \left(\frac{l_i}{Z}\right)^2\label{f1_expr}
\end{align}
\eqref{arcsinineq} follows from the fact that $h_i$ and $h_k$ belong to the same group $G_n$. So we have, $$\frac{Q'_{ik}}{\sqrt{1+Q'_{ii}} \sqrt{1+Q'_{kk}}} = \frac{1}{\sqrt{2}\sqrt{2}} = \frac{1}{2}$$
We can also see that \eqref{f1_expr} follows from Lemma \ref{converge-second} and the ansatz  $v_i = v_j$ when $i,j \in G_n$. The order parameters $T_{nm}$ doesn't change after pruning, and so we have,
\begin{align}
    f_2(T') & = \frac{1}{\pi}  \sum_{n,m} v_n^* v_m^* \arcsin \frac{T_{nm}}{\sqrt{1+T_{nn}} \sqrt{1+T_{mm}}},\nonumber\\
    & = \frac{1}{6} \sum_{n=1}^{M} (v^*)^2\label{f2_expr}
\end{align}
And similarly,
\begin{align}
    f_3(R',Q',T') & =  \frac{2}{\pi}  \sum_{i,n} v_i v_n^* \arcsin \frac{R'_{in}}{\sqrt{1+Q'_{ii}} \sqrt{1+T'_{nn}}},\nonumber\\
    & = \frac{2}{\pi} \sum_{n=1}^{M} v_n ^* \sum_{i \in P_n} v_i \arcsin \frac{1}{2},\nonumber\\
    & = \frac{2}{6} \sum_{n=1}^{M} v_n ^* \sum_{i \in P_n} v_i.\label{f3_expr}
\end{align}
Then from \eqref{f1_expr},\eqref{f2_expr} and \eqref{f3_expr} the GE of node pruning is,
\begin{equation}
\frac{(v^*)^2}{6}\left[ \sum_{m=1}^{M}\left(1-\frac{l_m}{Z} \right)^2\right].
\end{equation}
Hence we have the lemma.
\end{proof}

Intuitively, this lemma states that for teacher hidden node $t_n$ if  $l_n$ student hidden nodes survive after node pruning, then the fraction of information lost due to the deletion of  nodes is $1-\frac{l_n}{Z}$, where $Z$ is the number of student nodes learn a particular teacher node in the unpruned network.

\begin{lemma}\label{converge-second}
Let $v^*$ denotes the weight of the second layer of the teacher network and $\{v_1,\cdots,v_K\}$ be the weights of the student network after convergence. Then in the noiseless case for all $n$ we have, 
$$v^* = \sum_{i\in G_n} v_i$$ 
\end{lemma}

\begin{proof}[Proof of Lemma \ref{converge-second}]
From $(S36)$ of \cite{goldt2019dynamics} we have, 
\begin{align*}
    \frac{dv_i}{dt} & = \eta_v \left[ \sum_{n=1}^{M} v_n ^* I_2(i,n) -\sum_{j=1}^{K} v_j I_2(i,j) \right] \\
    & = \eta_v \arcsin \frac{1}{2} \left[  v^*  -\sum_{j \in G_{n}} v_j  \right]
\end{align*}
Hence, a fixed point (in terms of $v_i$'s) of the ODE is,
$$\{(v_1,\ldots,v_K) | \sum_{i \in G_n} v_i = v^* ,\forall 1\leq n\leq M)\}$$
\end{proof}

Intuitively, this lemma states that the sum of the outgoing edges of the student hidden nodes which learn a particular teacher hidden node is approximately equal to the weight of the outgoing edge of that teacher hidden node.

% \subsection{GE in Unpruned Network}
\begin{figure*}[t]
    \centering
    \includegraphics[width=\textwidth]{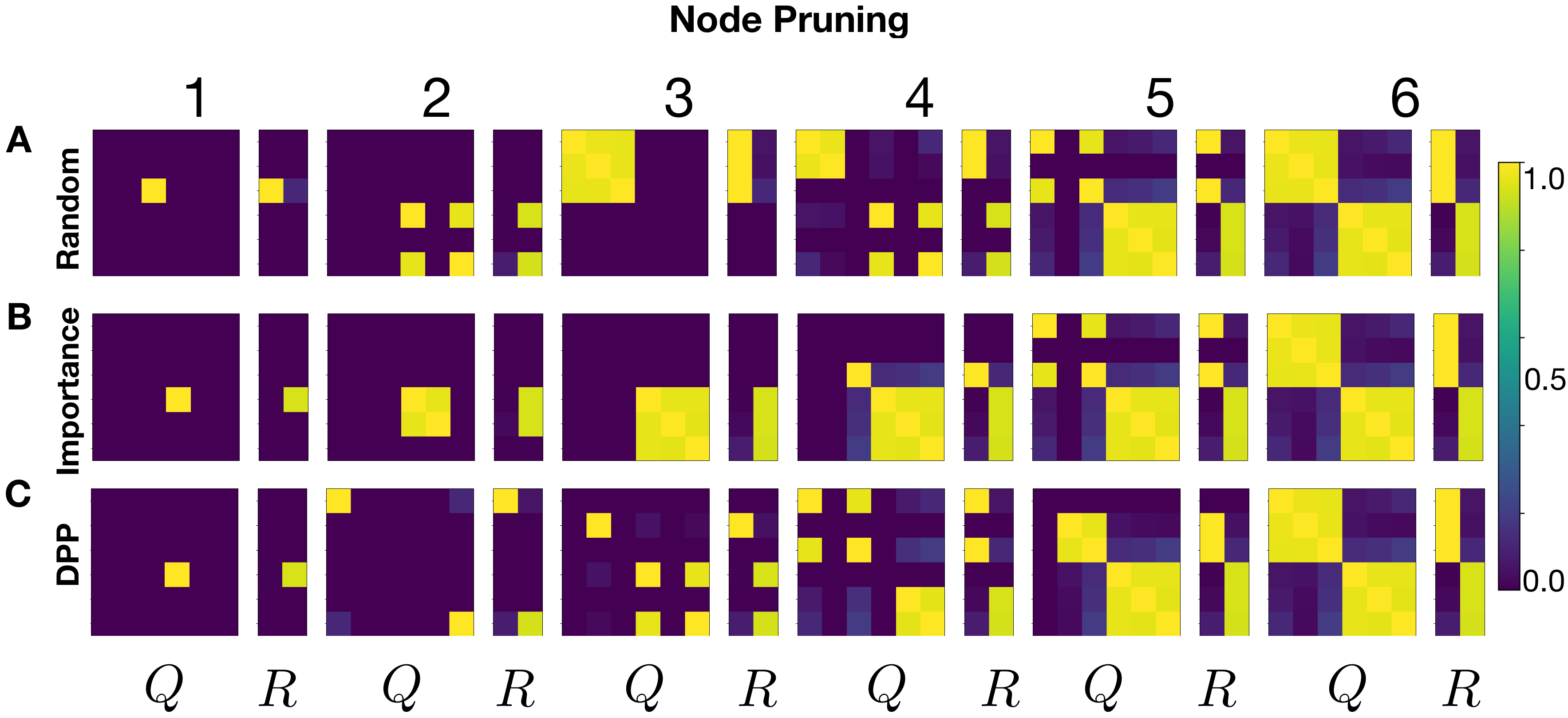}
    \caption{Order parameters after different node pruning methods in the teacher student setup. For this example, number of student hidden nodes $M=2$ and number of teacher hidden nodes $K=6$. From \cite{mariet2015diversity} we know that the first $3$ student nodes (call $1^{st}$ group) learn one teacher node, and the next $3$ (call $2^{nd}$ group) learn the second teacher node. Recall that $k_n$ is the number of student hidden nodes survived after node pruning. In this figure each row represents a particular node pruning method and each column (Q, R) shows results for different choices of $k_n$ (left to right goes from most pruned to unpruned network). \textbf{(A)} In case of random node pruning when $k_n=2$, two student node survives from the $2^{nd}$ group after pruning. As a result, information about the $1^{st}$ teacher node is completely lost in the pruned network. \textbf{(B)} Importance pruning keeps a student hidden node depending on its outgoing edge weights. The outgoing edge weights of each group is almost equally distributed among themselves, and they sum up to the second layer weight of corresponding teacher node (see Lemma \ref{converge-second}). As all the group size is equal (3 for this example), importance node pruning first selects node from the group whose corresponding second layer teacher weight is highest. In our example, it is the second group and hence for $k_n=1,2,3$, it selects node from the second group. Once a group is exhausted, it then selects from another group according to the aforementioned policy and so on. \textbf{(C)} For DPP node pruning when $k_n = 2$, two student hidden nodes are chosen from different groups which preserve information about both the teacher nodes. It can also be shown that, in case of node pruning, if at least one representative from a group survives after pruning, then the reweighting can recover the complete information about that block. Hence, in teacher student framework DPP node pruning performs the best among the node pruning methods especially after reweighting.}
    \label{fig:node-pruning-all}
\end{figure*}

\begin{figure*}[t]
    \centering
    \includegraphics[width=\textwidth]{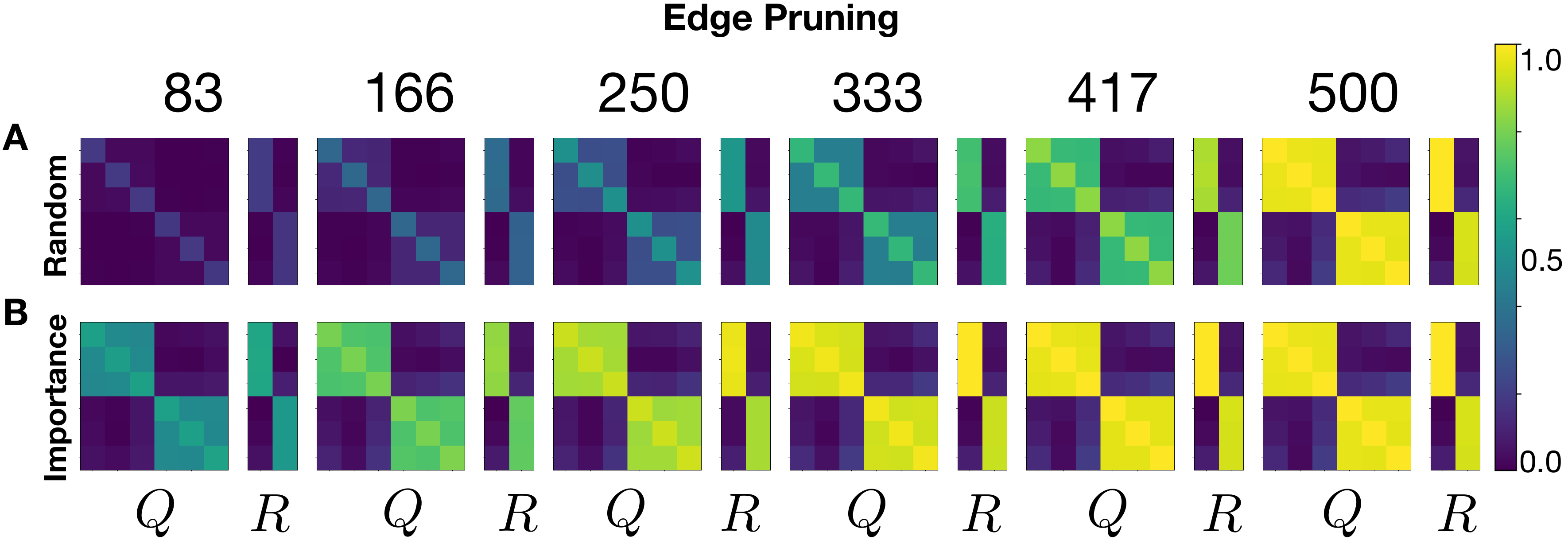}
    \caption{Order parameters after different edge pruning methods in the teacher student setup. For this example, number of student hidden nodes $M=2$ and number of teacher hidden nodes $K=6$. From \cite{mariet2015diversity} we know that the first $3$ student nodes (call $1^{st}$ group) learn one teacher node, and the next $3$ (call $2^{nd}$ group) learn the second teacher node. Recall that $k_e$ is the number of incoming edges for each student hidden nodes survived after edge pruning. In this figure each row represents a particular edge pruning method and each column (Q, R) shows results for different choices of $k_e$ (left to right goes from most pruned to unpruned network). \textbf{(A)} In case of random edge pruning, the expected order parameters have the form described in Lemma \ref{rand_edge-pruning}. \textbf{(B)} Order parameters for importance edge pruning. For importance edge pruning, the edges with lowest absolute values are removed. As the input dimension goes to infinity, the order parameters of the pruned network are close to that of the unpruned network ($k_e = 500$).  In particular, for any fix $k_e$, let $Q^{imp}_{k_e}$ be the order parameter of the pruned network when importance pruning is used. $Q^{rand}_{k_e}$ is defined similarly. Our simulations show that, $\norm{Q^{unpruned}-Q^{imp}_{k_e}} \leq \norm{Q^{unpruned}-Q^{rand}_{k_e}}$. This is why the blocks in the $Q$ matrix are the brightest in case of importance pruning. Hence, importance edge pruning performs the best without reweighting. }
    \label{fig:edge-pruning-all}
\end{figure*}

\begin{lemma}\label{rand_edge-pruning}
Let $Q,R,T$ are the order parameters of the unpruned network, and $Q',R',T'$ are the respective order parameters after applying the Random Edge pruning where $c$ fraction of the edges are kept. Then we have the following:
\begin{itemize}
\item \[
\mathbb{E}[Q'_{ik}] =
\begin{cases}
     c Q_{ik}  & \text{ if } i = k \\
    c^2 Q_{ik} & \text{ otherwise } 
\end{cases}
\]
    \item $\mathbb{E}[R'_{st}] = c R_{st}$ 
    \item $T'_{mn} = T_{mn} $
\end{itemize}
\end{lemma}

\begin{proof}
In case of Random Edge pruning each edge is kept with probability $c$. Then we have,
$$ \mathbb{E}[Q'_{st}]  =\frac{1}{N} \sum_{i=1}^N c  \cdot w_{is} \times c \cdot w_{it}  =    c^2 \frac{1}{N} \sum_{i=1}^N   w_{is} w_{it}  = c^2 Q_{st} $$
and
$$\mathbb{E}[Q'_{ss}] =\frac{1}{N} \sum_{i=1}^N c^2  \cdot w_{ss}^2= c^2 Q_{ss}.$$
Similarly,
$$\mathbb{E}[R'_{st}] =\frac{1}{N} \sum_{i=1}^N c \cdot w_{is} w^*_{it}  = c R_{st}
$$
The teacher node is not affected by the pruning. So $T$ is not modified by the pruning process. This proves the lemma.
\end{proof}

Intuitively, this lemma states that the order parameters of the pruned network using random edge pruning is a scaled version of the order parameters of the unpruned networks. However, the scaling of diagonal elements are different from that of off-diagonal elements (for more see Figure \ref{fig:edge-pruning-all} A).

\section{Simulation Details}\label{simulation_details}

In total, $10$ rounds of simulations are run for each of the 5 pruning methods, and we report the average and standard deviations (as error bars). The standard deviations are negligible (in the magnitude of $10^{-3}$). A \textit{round} is the entire process of generating a new teacher network with datasets, training the student from scratch, performing pruning and finally testing with the pruned network. For DPP and random methods, we sampled $100$ masks per round and reported the average performance in each round. Given $M = 2$ and  $K = 6$, we tried pruning with $[1, 2, 3, 4, 5]$ nodes (and the equivalent number of edges) left in the student, respectively. We keep the total number of weights same to compare different pruning methods. The node-to-edge ratio, given $N = 500$, $K = 6$, and $M = 2$, is [$1:83, 2:166, 3:250, 4:333, 5:417$]. This is calculated, for the teacher-student setup (single output node) specifically, as $k_e = \frac{k_n(1 + N) - K}{K}$. We grid-searched $\eta$ in the range of [$0.1, 0.2, 0.3, 0.4, 0.5, 0.6, 0.7$] and found $0.50$ to be the optimal. We used $\beta = 0.3$ for all DPP node kernel calculations in all simulations.

\section{Hyperparameters for Real Datasets}\label{hyperparam_real_data}
Besides the hyperparameters and setup we proposed in Section~\ref{Simulation} on the synthetic dataset, we report the hyperparameters used for the results on the~\texttt{MNIST} and~\texttt{CIFAR10}. As stated in Section~\ref{real-data-experiments}, we used that exact same experiment setup (network architectures, error thresholds, etc.) as in~\cite{mariet2015diversity} for fair and consistent comparisons. We used SGD optimizers, a learning rate of $0.001$, and a momentum of $0.9$ for traning on both datasets. For ~\texttt{MNIST}, the training batch size was $1000$. For ~\texttt{CIFAR10}, the training batch size was $128$.  All pruning methods were performed $10$ times, and we report the means and standard deviations in Figure~\ref{fig:data-results} (with reweighting). 

The node-to-edge ratio for pruning, which keeps the number of parameters in the pruned network the same, is $[397:614, 472:921, 548:1228, 623:1536, 699:1843, 774:2150, 849:2457, 925:2764]$ for \texttt{CIFAR10} and $[256:156, 287:235, 317:313, 348:392, 378:470, 409:548, 439:627, 470:705]$ for \texttt{MNIST}, given the network architecture in Table 1 of~\cite{mariet2015diversity}. These ratios correspond to $20\%$ to $90\%$ of the edges left for each node, as shown on the x-axis of Figure~\ref{fig:data-results}. These node-to-edge ratios are calculated based on the conversion equation in Section~\ref{real-data-experiments}. We used $\beta = 10 / |T|$ where $|T|$ is the size of the training dataset for all DPP node and edge kernel calculations on real data, following the choice of~\cite{mariet2015diversity}.

\begin{table*}[t]
\caption{The mean square GE on synthetic data for all pruning methods. The left-most row indicates the percentage of parameters left in the network. For specific node-to-edge ratio, see~\ref{node-edge-relation}. The upper table shows the noiseless case, and the lower shows the noisy case ($\sigma = 0.25$). 
%\textcolor{red}{change} In both cases, DPP Edge performs the best. 
We also observed the implicit regularization effects of pruning proposed by ~\cite{mariet2015diversity}}
\label{synthetic-loss-table}
\vskip 0.15in
\begin{center}
\begin{small}
\begin{sc}
\scalebox{1}{
\begin{tabular}{|c|c|c|c|c|c|c|}
\hline
\% of Parameters &  DPP Node & Rand. Edge & Rand. Node & Imp. Edge & Imp. Node\\
\hline
$17.0\%$    & 3.737$\pm$ 0.009& 3.451$\pm$ 0.011 & 3.978 $\pm$ 0.016 & 1.911 & 3.760 \\
$33.0\%$   &  2.310$\pm$ 0.012& 2.300$\pm$ 0.015 & 2.800 $\pm$ 0.035 & 0.814 & 2.719 \\
$50.0\%$   &  1.438$\pm$ 0.015& 1.402$\pm$ 0.006 & 1.748 $\pm$ 0.036 & 0.311 & 1.540 \\
$67.0\%$   &  0.740$\pm$ 0.017& 0.730$\pm$ 0.006 & 1.046 $\pm$ 0.018 & 0.110 & 0.721 \\
$83.0\%$   &  0.258$\pm$ 0.008& 0.204$\pm$ 0.005 & 0.540 $\pm$ 0.010 & 0.040 & 0.360 \\
\hline
\multicolumn{6}{c}{Original Test Loss: $0.051$ (Noiseless)} \\
\hline
$17.0\%$    &  4.000$\pm$ 0.005& 3.769$\pm$ 0.012 & 4.188 $\pm$ 0.001 & 1.963 & 4.167 \\
$33.0\%$   &  2.622$\pm$ 0.015& 2.558$\pm$ 0.011 & 3.041 $\pm$ 0.024 & 0.905 & 2.910 \\
$50.0\%$   &  1.633$\pm$ 0.002& 1.675$\pm$ 0.010 & 2.023 $\pm$ 0.035 & 0.450 & 2.031 \\
$67.0\%$   &  0.890$\pm$ 0.018& 1.007$\pm$ 0.007 & 1.269 $\pm$ 0.022 & 0.271 & 1.144 \\
$83.0\%$   &  0.394$\pm$ 0.001& 0.490$\pm$ 0.003 & 0.643 $\pm$ 0.002 & 0.253 & 0.659 \\
\hline
\multicolumn{6}{c}{Original Test Loss: $0.241$ ($\sigma = 0.25$)} \\
\end{tabular}}
\end{sc}
\end{small}
\end{center}
\vskip -0.1in
\end{table*}

\section{Tables and Figures}
Table~\ref{synthetic-loss-table} shows the experimental results on the synthetic data with the setup discussed in main text. For all the node-to-edge ratios in~\eqref{node-edge-relation}, given $K = 6$ and $M = 2$, we calculated the mean square GEs for both the noiseless and noisy case ($\sigma = 0.25$). We sampled $100$ masks per simulation, and there are in total 10 rounds of simulations. As mentioned earlier, DPP methods are stable, and the standard deviations are in the magnitude of $10^{-3}$ for all ratios.

% \bibliography{bibfile}

\end{document}